\documentclass{article} 
\usepackage{nips15submit_e,times}
\usepackage{hyperref}
\usepackage{url}
\usepackage[pdftex]{graphicx}

\usepackage{mathtools}
\usepackage{amsthm}
\usepackage{amsmath}
\usepackage{amssymb}
\usepackage{amsfonts}
\usepackage{setspace}

\newtheorem{theorem}{Theorem}
\newtheorem{lemma}{Lemma}

\theoremstyle{definition}
\newtheorem{definition}{Definition} 

\title{One-Class Semi-Supervised Learning:\\
Detecting Linearly Separable Class by its Mean}

\author{
Evgeny Bauman\\\
Markov Processes Inc.\\
\texttt{ebauman@markovprocesses.com} \\
\And
Konstantin Bauman \\
Stern School of Business \\
New York University \\
\texttt{kbauman@stern.nyu.edu}
}

\nipsfinalcopy 

\begin{document}

\maketitle

\begin{abstract}

In this paper, we presented a novel semi-supervised one-class classification algorithm which assumes that class is linearly separable from other elements. We proved theoretically that class is linearly separable if and only if it is maximal by probability within the sets with the same mean. Furthermore, we presented an algorithm for identifying such linearly separable class utilizing linear programming. We described three application cases including an assumption of linear separability, Gaussian distribution, and the case of linear separability in transformed space of kernel functions. 
Finally, we demonstrated the work of the proposed algorithm on the USPS dataset and analyzed the relationship of the performance of the algorithm and the size of the initially labeled sample.

\end{abstract}


\section{Introduction}\label{intro}

In machine learning, one-class classification problem aims to identify elements of a specific class among all other elements. This problem has been extensively studied in the last decade \cite{Khan:2009:SRT:1939047.1939070} and the developed methods were applied to a large variety of problems, such as detecting outliers \cite{Scholkopf:2001:ESH:1119748.1119749}, natural language processing \cite{Joffe962}, fraud detection \cite{7435726}, and many others \cite{Khan:2009:SRT:1939047.1939070}.

The traditional supervised approach to the classification problems infers a decision function based on labeled data. In contrast, semi-supervised learning deals with the situation where relatively few labeled training elements are available, but a large number of unlabeled elements are given. This approach is suitable for many practical problems where is it relatively expensive to produce labeled data, such as automatic text classification. Semi-supervised learning may refer to either {\it inductive} learning or {\it transductive} learning \cite{6280886}. The goal of {\it inductive} learning is to learn a decision rule that would predict correct labels for the new unlabeled data, whereas {\it transductive} learning aims to infer the correct labels for the given unlabeled data only. 
In this paper we address the semi-supervised one-class transductive learning problem, i.e. the main goal of the proposed algorithm is to estimate the labels for the given unlabeled data. 

In order to make semi-supervised learning work, one has to assume some structure in the underlying distribution of data. The most popular such assumptions include smoothness assumption, cluster assumption, or manifold assumption \cite{Chapelle:2010:SL:1841234}. In the proposed algorithm we make an assumption of linear separability of the class. Therefore, we study the question of what information is needed to estimate the linearly separable class. As a result, we prove that linearly separable class can be detected based on its mean.


In this paper we made the following contributions:
\begin{itemize}
\item Proposed a new semi-supervised approach to one-class classification problem.
\item Proved that class is linearly separable if and only if it's maximal by probability within all sets with the same mean.
\item Presented an exact algorithm for detecting linearly separable class by its mean, utilizing linear programming.
\item Described three application cases including an assumption of linear separability of the class, Gaussian distribution, and the case of linear separability in the transformed space of kernel functions. 
\item Demonstrated the work of the proposed algorithm on the USPS dataset and analyzed the relationship between the performance of the algorithm and the size of the labeled sample.
\end{itemize}

\section{Related work}\label{prior}

There are two main blocks of the related work. The first block consists of works pertaining to one-class classification problem. The most popular supervised support vector machine approach to the one-class problem (OC-SVM) was independently introduced by Sch\"{o}lkopf et al. in \cite{NIPS1999_1723} and Tax and Duin in \cite{Tax2004}. This studies extended the SVM methodology to handle training consisting of labels of only one class. In particular, \cite{NIPS1999_1723} proposed an algorithm which computes a binary function that captures regions in input space where the probability density lives (its support), i.e. a function such that most of the data will live in the region where the function is nonzero.
Later, OC-SVM was successfully applied to the anomaly and outliers detection \cite{PMID:21346963, 4217460, Manevitz:2002:OSD:944790.944808, Zhang:2007:OCS:1503549.1503556}, and the density estimation problem \cite{Munoz2004}. This approach was effective in remote-sensing classification in the situations when users are only interested in classifying one specific land-cover type, without considering other classes \cite{5559411}. Authors of \cite{Zhang:2007:OCS:1503549.1503556} applied OC-SVM approach to detect anomalous registry behavior by training on a dataset of normal registry accesses. 
In \cite{PMID:21346963} authors use OC-SVM as a means of identifying abnormal cases in the domain of melanoma prognosis. More applications of OC-SVM can be found in \cite{Khan:2009:SRT:1939047.1939070}.

The second block of the related works consists of semi-supervised approaches to the classification problems. Book \cite{Chapelle:2010:SL:1841234} provides an extensive review of this field and \cite{DBLP:journals/corr/PrakashN14} presents a survey describing the current state-of-the-art approaches. 
There are several works that study one-class semi-supervised learning \cite{Amer:2013:EOS:2500853.2500857, RePEc:pal:jorsoc:v:64:y:2013:i:4:p:513-529, 5559411, 5460897}. In particular, authors of \cite{RePEc:pal:jorsoc:v:64:y:2013:i:4:p:513-529} evaluate the suitability of semi-supervised one-class classification algorithms as a solution to the low default portfolio problem. In \cite{5460897} authors implemented modifications of OC-SVM in the context of information retrieval. 
All these works utilize various modifications of OC-SVM approach. Note, that in this cases the optimization problem utilizes quadratic programming. 

In contrast to this prior works, we propose a semi-supervised one-class classification algorithm that utilizes linear programming for optimization. Our algorithm is transductive and mostly focused on estimating labels for the given unlabeled data.

\section{Algorithm}\label{algorithm}


We consider probability space $(\mathbb{R}^n,\Sigma,P)$, where $\mathbb{R}^n$ is Euclidean space,   $\Sigma$ is a Borel $\sigma$-algebra over $\mathbb{R}^n$ (it contains all half-spaces of $\mathbb{R}^n$), and $P$ is a probability measure in $\mathbb{R}^n$.
We assume that $(\mathbb{R}^n,\Sigma,P)$ satisfies the following conditions:
\begin{enumerate}
\item{There exists a finite second moment $\int x^Tx dP(x)$}\label{cond1}
\item{The probability of any hyperplane in $\mathbb{R}^n$ is not equal to $1$.}\label{cond2}
\end{enumerate}

Set $A\in \Sigma$ is defined by its indicator function
\[h_A(x) = \left\{
  \begin{array}{lr}
    1, & x \in A\\
    0, & x \notin A
  \end{array}
\right.
\]
Further, in the paper instead of the sets, we work with the indicator functions defining them and, therefore, we also define concepts of ``linear separability" and ``maximum by probability" for indicator functions.

Based on Condition \ref{cond1}, indicator functions $h_A\in L^2(\mathbb{R}^n,\Sigma,P)$.
We denote by $\mathcal{H} = \{h_A|A\in\Sigma\}$ -- the set of indicator functions for all measurable sets. The convex hull of $\mathcal{H}$ would be the set of functions 
$$\hat{\mathcal{H}} =Conv(\mathcal{H})=\{h\in L^2(\mathbb{R}^n,\Sigma,P)|
\forall x: ~ 0 \leq h(x) \leq 1\}.$$
By construction $\mathcal{H}\subseteq\hat{\mathcal{H}}$. Elements of $\hat{\mathcal{H}}$ usually interpreted as fuzzy sets, saying that set defined by $h$ contains element $x$ with the weight $h(x),$ where $0 \leq h(x) \leq 1$.


According to Condition \ref{cond1} for each indicator function $h$ defining measurable fuzzy set   with the probability greater that $0$, we can construct its mean (i.e., first normalized moment of the set defined by $h$):

\begin{equation}
\mu(h) = \frac{M(h)}{P(h)} = \frac{\int x h(x)dP(x)}{\int h(x)dP(x)},
\end{equation}

where $P(h) = \int h(x) dP(x)$ -- scalar equal to the probability of $h$, and $M(h) = \int x h(x) dP(x)$ -- $n$-dimensional vector, equal to first un-normed moment of $h$.

Further, we construct linear mapping of $\mathcal{H}$ to $(n+1)$-dimensional Euclidean space:
\begin{equation}
\varphi(h) = \left(M(h),P(h)\right).
\end{equation}

$\mathcal{M} = \varphi(\mathcal{H})$ is the image of the set $\mathcal{H}$ with mapping $\varphi$ to $\mathbb{R}^{n+1}$. 
Since $\varphi$ is linear, then $\hat{\mathcal{M}} = Conv(\mathcal{M}) = \varphi(\hat{\mathcal{H}}).$

Based on the fact that $\hat{\mathcal{H}}$ is convex and closed set, and Conditions \ref{cond1} and \ref{cond2} we conclude that $\hat{\mathcal{M}}$ is also convex, closed and bounded. Moreover, $\hat{\mathcal{M}}$ does not belong to any hyperplane in $\mathbb{R}^{n+1}$.


\subsection{Linear Separability and Maximum by Probability}
\begin{definition}\label{h_linearly_separable} 
Indicator function $h \in \hat{\mathcal{H}}$ is {\it linearly separable} with  vector $\overline{b} = (\overline{c}, d)$, where $\overline{c} \in \mathbb{R}^{n}$ ($\overline{c}$ is non-trivial) and $d$ is a constant, if 
the following conditions are satisfied almost everywhere:
\begin{itemize}
\item if $c^Tx+d > 0$, then $h(x) = 1$
\item if $c^Tx+d < 0$, then $h(x) = 0$.
\end{itemize}
\end{definition}

Note, that Definition \ref{h_linearly_separable} implies that  that if $h \in \hat{\mathcal{H}}$ is linearly separable then $P(\{x\in \mathbb{R}^{n} |c^Tx+d\neq0\}\cap\{x\in \mathbb{R}^{n}| 0 < h(x) < 1\})=0$. It means that fuzziness of $h$ concentrates only on the hyperplane defined by equation $c^Tx + d = 0$.

\begin{lemma}\label{lemma} 

Image $y^*= \varphi(h^*)$ of function $h^*\in \hat{\mathcal{H}}$ (with $P(h^*) \neq 0$), is a boundary point of $\hat{\mathcal{M}}$ 
if and only if 
there exists a non-trivial vector $\bar{b}$, such that $h^*$ is linearly separable with $\bar{b}$. 


\end{lemma}

\begin{proof}
Based on the Supporting hyperplane theorem \cite{Boyd:2004:CO:993483}, 
$y^*$ is a boundary point of convex closed set $\hat{\mathcal{M}}$ if and only if there exists a non-trivial vector $\bar{b} = (\bar{c},d)$, where $\bar{c} \in \mathbb{R}^{n}$ and $d$ is a constant, such that $b^T y^* \geq b^T y,~\forall y \in \hat{\mathcal{M}}$,
which is equivalent to $b^T \varphi(h^*) \geq b^T \varphi(h),~\forall h \in \hat{\mathcal{H}}$.
If we denote by $\Delta_b (h_1,h_2) = b^T \varphi(h_1) - b^T \varphi(h_2)$, then the inequality transforms to 
\begin{equation}\label{lemmequation}
\Delta_b (h^*,h) \geq 0, ~\forall h \in \hat{\mathcal{H}}.
\end{equation}

{\bf Sufficiency.}
If $h^*$ is linearly separable with $\bar{b}= (\bar{c},d)$, then for arbitrary $h\in \hat{\mathcal{H}}$ we consider the difference:
$$\Delta_b(h^*,h) = b^T \varphi(h^*) - b^T \varphi(h)  = \left(c^T M(h^*) + d P(h^*)\right) - \left(c^T M(h) + d P(h)\right)$$
that can be transformed to 
$$\Delta_b(h^*,h) = \int(c^T x+d)(h^* (x) - h (x)) dP(x)  = \Delta^+_b(h^*,h) +\Delta^0_b(h^*,h) +\Delta^-_b(h^*,h), \mbox{ where}$$
$$\begin{array}{l}
\Delta^+_b(h^*,h) = \int_{c^T x+d>0}(c^T x+d)(1-h (x))dP(x),\\
\Delta^0_b~(h^*,h) = \int_{c^T x+d=0}(c^T x+d)(1-h (x))dP(x) = \int_{c^T x+d=0}0 \cdot (1-h (x))dP(x) = 0,\\
\Delta^-_b(h^*,h) = \int_{c^T x+d<0}(c^T x+d)(0-h (x)) dP(x).
\end{array}$$

Since $0 \leq h (x) \leq 1~\forall x$, then $\Delta^+_b(h^*,h)\geq 0$ and $\Delta^-_b(h^*,h) \geq 0$, and, therefore, $\Delta_b(h^*,h) \geq 0 ~\forall h \in \hat{\mathcal{H}}$. It shows that there exists $\bar{b}$ that satisfies inequality (\ref{lemmequation}) and implies that $\varphi(h^*)$ is a boundary point of $\hat{\mathcal{M}}$.

{\bf Necessity.}
If $y^*= \varphi(h^*)$ is a boundary point of $\hat{\mathcal{M}},$ then there exists $\bar{b}$ satisfying inequality (\ref{lemmequation}). 
Vector $\bar{b} = (\bar{c},d)$ defines the half-space in $\mathbb{R}^{n}$ with the following indicator function:
\[h_{b}(x) = \left\{
  \begin{array}{lr}
    1, & \mbox{ if }~ c^Tx + d \geq 0,\\
    0, & \mbox{ if }~ c^Tx + d < 0.
  \end{array}
\right.
\]

Based on the same calculation as in proof of Sufficiency: $\Delta_b(h_b, h) \geq 0~ \forall h\in \hat{\mathcal{H}}$, and thus $\Delta_b(h_b, h^*) \geq 0$. However, inequality (\ref{lemmequation}) implies that $\Delta_b(h_b, h^*) \leq 0$, which means that $\Delta_b(h_b, h^*) = 0.$ Therefore, $\Delta^+_b(h_b,h^*) = 0$, and $\Delta^-_b(h_b,h^*) = 0$, that can be expressed in 
$$\int_{c^T x+d>0}(1-h^* (x)) dP(x) =\int_{c^T x+d<0} (0-h^* (x)) dP(x) =0.$$.

This equality implies that the following conditions satisfied almost everywhere:
\begin{itemize}
\item if $c^Tx+d > 0$, then $h^*(x) = 1$
\item if $c^Tx+d < 0$, then $h^*(x) = 0$,
\end{itemize}
and, therefore, $h^*$ is linearly separable with $\bar{b}= (\bar{c},d).$
\end{proof}

\begin{definition}
Let's denote by $W_\alpha$ a set of all indicator functions that define sets in $\mathbb{R}^n$ with the means in $\alpha$, i.e. $W_\alpha = \{h \in \hat{\mathcal{H}}: \mu(h)=\alpha\}$.
Indicator function $h^* \in W_\alpha$ is called {\it maximal by probability with mean $\alpha$} if $\forall h\in W_\alpha: P(h) \leq P(h^*)$. 
\end{definition}


\begin{lemma}\label{lemmaInterval} 
$\varphi(W_\alpha)$ is a finite interval in $\mathbb{R}^{n+1}$, i.e. for any $h^*\in W_\alpha$ if $\varphi(h^*) = (M(h^*),P(h^*)) = y^*$ then $\varphi(W_\alpha) = \{y = \lambda\cdot y^* | 0 < \lambda \leq \lambda_{max}^*\}$, where $\lambda_{max}^* \geq 1$.
\end{lemma}

\begin{proof}

For an arbitrary $h^* \in W_\alpha$ with $\varphi(h^*) = (m^*,p^*) = y^*$, we construct a ray $D_{h^*} = \{y = \lambda\cdot y^* | \lambda > 0\}$ in $\mathbb{R}^{n+1}$.
Since $\hat{\mathcal{M}}$ is convex, closed and bounded, then intersection of $D_{h^*}$ and  $\hat{\mathcal{M}}$ is a finite interval $I_{h^*} = D_{h^*} \cap \hat{\mathcal{M}} = \{y = \lambda\cdot y^* | 0 < \lambda \leq \lambda_{max}^*\}$. Since $y^* \in I_{h^*}$, then $\lambda_{max}^* \geq 1$.

The statement $h \in  W_\alpha$ means that $\mu(h) = \frac{M(h)}{P(h)} = \alpha = \frac{M(h^*)}{P(h^*)} = \mu(h^*)$. If we denote $\frac{P(h)}{P(h^*)} = \lambda$, then $M(h) = \lambda\cdot M(h^*)$. Which is equivalent to the fact that there is exists a certain $0 < \lambda \leq \lambda_{max}^*$, such that $\varphi(h) = (m, p) =  (\lambda \cdot m^*, \lambda \cdot p^*) = \lambda \cdot y^* $. This means that $I_{h^*} = \varphi(W_\alpha)$.

\end{proof}

\begin{theorem}\label{equalityTheorem}
Function $h^* \in W_\alpha$ is maximal by probability, 
if and only if $h^*$ is linearly separable.
\end{theorem}

\begin{proof}

{\bf Necessity.}
Based on Lemma \ref{lemmaInterval}, $\varphi(W_\alpha) = \{y = \lambda\cdot y' | 0 < \lambda \leq \lambda_{max}'\}$ for an arbitrary $h'\in W_\alpha$.  
If $h^*\in W_\alpha$ is maximal by probability, then $P(h^*) = \lambda^*\cdot P(h') = \lambda_{max}\cdot P(h')$. Therefore, $\varphi(h^*) = y^* = \lambda_{max}'\cdot y'$, which means that $\varphi(h^*)$ is a boundary point of set $\hat{\mathcal{M}}$. Finally, based on Lemma \ref{lemma}, $h^*$ is linearly separable.

{\bf Sufficiency.}
If $h^*$ is linearly separable, then based on Lemma \ref{lemma}, $\varphi(h^*) = y^*$ is a boundary point of set $\hat{\mathcal{M}}$. 
Based on Lemma \ref{lemmaInterval}, $\varphi(W_\alpha) = \{y = \lambda\cdot y^* | 0 < \lambda \leq \lambda_{max}^*\}$.  Since $y^*$ is a boundary, then $\lambda_{max}^* = 1$. Therefore, $\forall h\in W_\alpha: ~ P(h) = \lambda\cdot P(h^*)$ where $0< \lambda \leq 1$. This means that $h^*$ is a maximal by probability with mean in $\alpha$.

\end{proof}

\subsection{Detecting Linearly Separable Class by its Mean}

Let $x_1,\dots,x_N$ be i.i.d. random variables in $\mathbb{R}^n$ with distribution $P$. The special case where $P$ is the empirical distribution $P_N (A) = \frac{1}{N}\sum_{i=1}^N h_A (x_i)$.

Based on Theorem \ref{equalityTheorem}, if class $A$ is linearly separable it is also maximal by probability within all sets with the same mean $\alpha=\mu(A)$. Therefore, the problem of detecting class $A$ would be equivalent to the problem of identifying the maximal by probability indicator function $h\in \hat{\mathcal{H}}$ with its mean in a given point $\alpha\in\mathbb{R}^n$, where the mean and the probability of $h$ are calculated in the following way: $M(h)=\frac{1}{N}\sum_{i=1}^N x_i h(x_i)$, $P(h)=\frac{1}{N}\sum_{i=1}^N h(x_i)$. In other words, in this case, we are searching for the set of elements with the specified mean $\alpha$ and containing the maximum number of elements. 
More formally we solve the following problem:

\vspace{-1.2cm}
{\setstretch{2.0}
\begin{equation}\label{max_prob_problem}
\left\{
  \begin{array}{l}
    P(h)=\frac{1}{N}\sum_{i=1}^N h(x_i) \xrightarrow[{h(x_1 ),\dots,h(x_N)}]{} max;\\
    \mu(h) = \frac{\sum_{i=1}^N x_i \cdot h(x_i)}{\sum_{i=1}^N h(x_i)} = \alpha;\\
    0\leq h(x_i ) \leq 1, ~ i=1,\dots,N.
  \end{array}
\right.
\Leftrightarrow
\left\{
  \begin{array}{l}
    \sum_{i=1}^N h(x_i) \xrightarrow[{h(x_1 ),\dots,h(x_N)}]{} max;\\
    \sum_{i=1}^N (x_i - \alpha) \cdot h(x_i) = 0;\\
    0\leq h(x_i ) \leq 1, ~ i=1,\dots,N.
  \end{array}
\right.
\end{equation}
}

\vspace{-0.2cm}
This problem can be solved using linear programming optimizing by vector $(h(x_1),\dots,h(x_N ))$. If there is exists at least one set with a given center $\alpha$, then Problem (\ref{max_prob_problem}) would have feasible solution.

In conclusion, we proved the equivalence of the concepts of ``linear separability'' and ``maximal by probability with a given mean", and proposed the exact algorithm of detecting linearly separable class based on its mean using linear programming.


\subsection{Algorithms of Semi-Supervised Transductive Learning for One-Class Classification}


{\bf Problem.}
Set $X =\{x_1,\dots,x_N\}\subset \mathbb{R}^n$ consists of two parts: sample $X_A=\{x_1,\dots,x_l\}\subset X \cap A$ containing labeled elements for which we know that they belong to class $A$, 
and the rest of the set $X_{unlabeled} = \{x_{l+1},\dots,x_N\}\subset X$ containing unlabeled elements for which we do not know if they belong to class $A$ or not. The problem is to identify all elements in $X$ that belong to class $A$ based on this information.


\subsubsection{Linearly Separable Class}\label{linearlySeparable}

In real practical problems we usually do not have the exact value of $\mu(A)$ and, therefore, we have to estimate it.
The proposed approach assumes that $X_A$ is generated with the same distribution law as $A$ and class $A$ is {\it linearly separable} in $\mathbb{R}^n$. 
Based on the first assumption, mean of $A$ is estimated by the mean of $X_A$ calculated as follows:$\mu(X_A)=\frac{1}{l}\sum_{j=1}^l x_j$. 
Based on the assumption of linear separability of $A$ it can be estimated by solving linear programming Problem (\ref{max_prob_problem}) with $\alpha=\mu(X_A)$. 
 More formally, we solve the following problem:

\vspace{-0.8cm}
{\setstretch{2.0}
\begin{equation}\label{linearlySeparableCase}
\left\{
  \begin{array}{l}
    \sum_{i=1}^N h(x_i) \xrightarrow[{h(x_1 ),\dots,h(x_N)}]{} max;\\
    \sum_{i=1}^N \left( x_i -\frac{\sum_{j=1}^l x_j}{l}\right) h(x_i) = 0;\\
    0\leq h(x_i ) \leq 1, ~ i=1,\dots,N.
  \end{array}
\right.
\end{equation}
}

\vspace{-0.2cm}
Since $\mu(X_A)$ is an estimation of the $\mu(A)$, in calculations we replace equalities $f(h)=0$ by two inequalities $-\varepsilon \geq f(h) \geq \varepsilon$ for a certain small value of $\varepsilon >0$.

Note that the performance of such estimation depends on two following factors:
\vspace{-0.3cm}
\begin{itemize}
\item The real linear separability of class $A$ in space $\mathbb{R}^n$.
\vspace{-0.1cm}
\item The difference between the mean estimated by $\mu(X_A)$ and the real mean $\mu(A)$.
\end{itemize}

In order to show how the proposed method works in practice, we constructed a synthetic experiment presented in Figure \ref{fig1} (left). In particular, it shows 300 points belonging to class $A$ (red crosses, yellow triangles, and blue pluses) and 750 points which do not belong to $A$ (purple crosses and gray circles). Class $A$ is linearly separable in $\mathbb{R}^2$ by construction. Initially we have 100 labeled points in set $X_A$ (red crosses on Fig.\ref{fig1} (left)) and we calculate its mean $\mu(X_A)$. Using the proposed method we construct the estimation $\hat A\in W_\alpha$ of class $A$ (yellow triangles and purple crosses Fig.\ref{fig1} (left)). As a result just $1\%$ of elements from class $A$ were classified incorrectly ($False ~Negative$), making $Recall = 99\%$. Moreover, there are no elements in $False ~Positive$ group, which makes the $Precision =100\%$.

\begin{figure}
\centering
\includegraphics[width=0.99\linewidth]{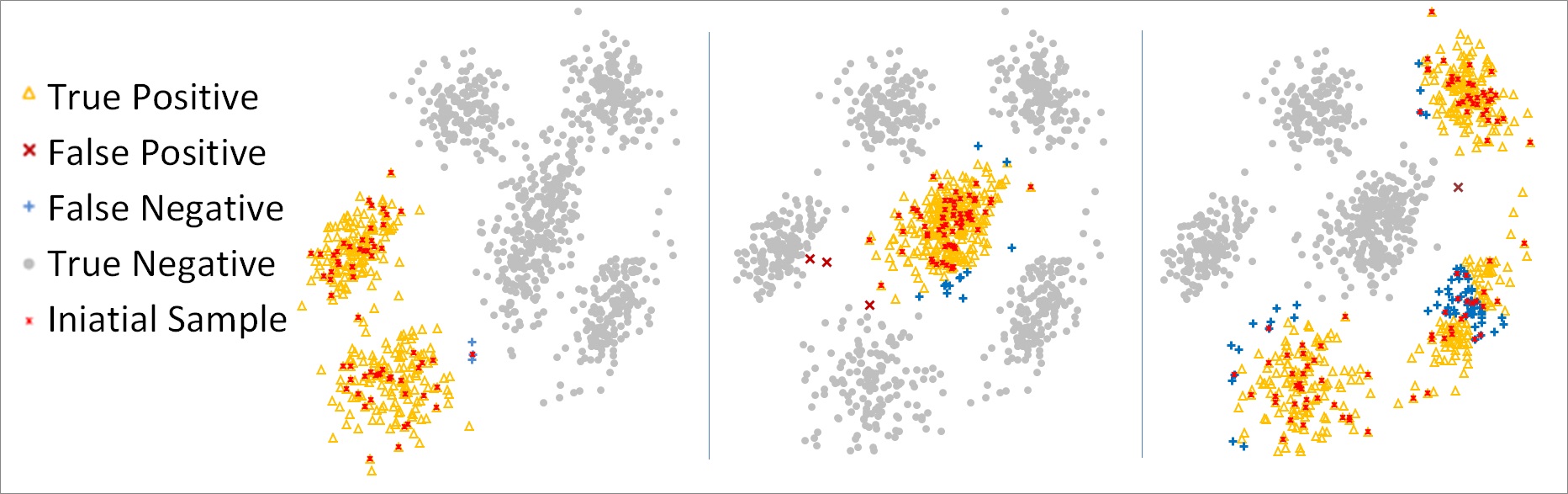}

\caption{Synthetic examples of (left) linearly separable class; (middle) normally distributed class; (right) linearly separable in transformed kernel space class.}\label{fig1}
\end{figure}


\subsubsection{Class with Gaussian Distribution}

In a certain class of problems, we can assume that elements of class $A$ have a Gaussian distribution. This distribution is defined by its mean and covariation matrix, and, therefore, class $A$ would fit an ellipsoid in $\mathbb{R}^n$ where most of the elements from $A$ are inside the ellipsoid and most of the others are outside. We can estimate such class $A$ by fixing not only the mean of $A$ but also its covariation matrix, i.e. every second moment of the class. We do it by enriching the space $\mathbb{R}^n$ with all pairwise products between the $n$ initial coordinates (features). In this case Problem (\ref{max_prob_problem}) transforms to the following form:

\vspace{-0.5cm}
{\setstretch{2.0}
\begin{equation}\label{normal_class}
\left\{
  \begin{array}{l}
    \sum_{i=1}^N h(x_i) \xrightarrow[{h(x_1 ),\dots,h(x_N)}]{} max;\\
    \sum_{i=1}^N \left( x_i -\frac{\sum_{j=1}^l x_j}{l}\right) h(x_i) = 0;\\
    \sum_{i=1}^N \left( x_i^{(t)}\cdot x_i^{(r)} -\frac{\sum_{j=1}^l x_j^{(t)}\cdot x_j^{(r)}}{l}\right) h(x_i) = 0;~t=1,\dots,n;~r=1,\dots,n;\\
    0\leq h(x_i ) \leq 1, ~ i=1,\dots,N.\\
  \end{array}
\right.
\end{equation}
}
\vspace{-0.3cm}

In this case, the resulting class would be separable in $\mathbb{R}^n$ with a surface of the second order.

Figure \ref{fig1} (middle) illustrates the work of the proposed method in the case of normally distributed class. All colors have the same meaning as in Figure \ref{fig1} (left) described in Section \ref{linearlySeparable}. There are 300 elements that belong to class $A$ and 750 elements that do not belong to $A$. In this example class, $A$ has a Gaussian distribution. Based on the initial sample $X_A$ consisting of 100 elements we calculate the estimation of the mean and the covariation matrix of $A$ and construct the estimation of $A$ by solving the Problem \ref{normal_class} using linear programming. The results of this estimation are presented in Figure \ref{fig1} (middle). In this case we got $Precision=96\%$ and $Recall=98\%$.


\subsubsection{Linearly Separable Class in the Transformed Space of Kernel Functions}

One of the ways of identifying the space where the class is linearly separable is introducing kernel function $K(x_i, x_r )$ \cite{Aizerman67theoretical}. The linear programming Problem \ref{linearlySeparableCase} for identifying the linearly separable set in the transformed space states as follows: 

\begin{equation}\label{kernel_class}
\left\{
  \begin{array}{l}
    \sum_{i=1}^N h(x_i) \xrightarrow[{h(x_1 ),\dots,h(x_N)}]{} max;\\
    \sum_{i=1}^N \left( K(x_i, x_r) -\frac{\sum_{j=1}^l K(y_j,x_r)}{l}\right) h(x_i) = 0;~r=1,\dots,N;\\
    0\leq h(x_i ) \leq 1, ~ i=1,\dots,N.\\
  \end{array}
\right.
\end{equation}

In order to illustrate the use of the proposed method in the transformed space of kernel functions, we construct an experiment where 450 elements belong to class $A$ and 600 elements do not (Fig.\ref{fig1} (right)). We optimize the indicator function $h$ in the Problem (\ref{kernel_class}) within the sets having the same mean in the transformed space as of the initial sample $X_A$ consisting of 150 elements from $A$. As a result of this experiment we got $Precision=94.5\%$ and $Recall=94.6\%$.




\section{Experiments}\label{experiments}

In order to show how the proposed algorithm works in practice, we applied it to the problem of classification of the handwritten digits in USPS dataset. This dataset was used previously in many studies including \cite{Scholkopf:2001:ESH:1119748.1119749}. It contains 9298 digit images of size $16 \times 16 = 256$ pixels. 
In our experiment, we tried to identify the class of digit ``0'' within the rest of the digits. We denote the class of ``0'' by $A$. There are 1553 elements of class $A$ in the dataset. 

First, we ran linear SVM algorithm in the original 256-dimensional space and it had an error in the classification of only one element, which means that class $A$ is almost linearly separable in this space. Secondly, we applied the proposed algorithm. We identified the real mean $\mu(A)$ of class $A$ and used it to find the maximal by probability set with the same mean, as defined in Problem (\ref{max_prob_problem}). Based on the Theorem \ref{equalityTheorem} this estimation $\bar{A}$ is linearly separable. The results show that $\bar{A}$ differs from $A$ in only one element. It means that our approach works well if we use the real mean of class $A$ for the calculations.
 
However, in practice we can only estimate the real mean of class $A$ based on a small sample of elements $X_A \subset A$ and solve Problem (\ref{linearlySeparableCase}) as described in Section \ref{linearlySeparable}. Therefore, in the next experiment we study the relationship between the size of initial sample $X_A$ an the performance of the constructed classification in terms of $precision$ and $recall$ measures. In particular, we run 100 times the following experiment:
\begin{enumerate}
\item Divide the dataset randomly to $T_{3100}$ and $S_{6198}$ with $3100$ and $6198$ of digit images correspondingly. The average number of elements from class $A$ in $S_{6198}$ is equal to 1046.
\item Within $T_{3100}\cap A$ create 20 random samples $\{X_A^{25}, X_A^{50}, X_A^{75}, \dots, X_A^{500}\}$ with sizes from 25 to 500 with a step of 25.
\item For each sample $X_A^*$ calculate its mean $\mu(X_A^*)$ and solve the linear programming Problem (\ref{linearlySeparableCase}) for the set $S_{6198}$ and mean $\mu(X_A^*)$.
\item Calculate the $precision$ and $recall$ measures of classification performance of elements in $S_{6198}$.
\end{enumerate}

Note, that in this experiment we used separate set $X_A^*$ that is not included into $S_{6200}$. We did it in order to eliminate the growth of the performance with the size of $X_A^*$ based on the larger number of known labels and not on the better estimation of the mean of class $A$.

\begin{figure}
\centering
\includegraphics[width=0.48\linewidth]{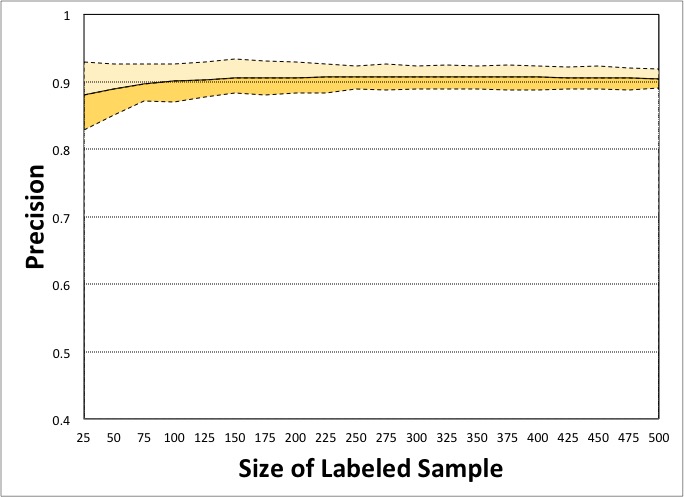}
\includegraphics[width=0.48\linewidth]{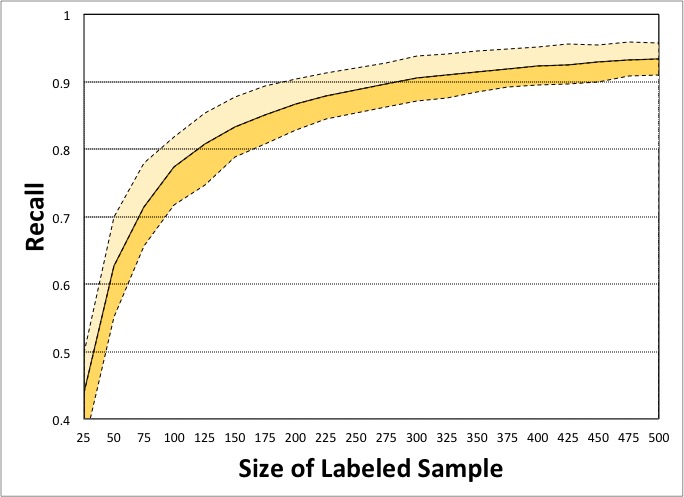}
\caption{Precision (left) and Recall (right) as functions from the size of the training sample $X_A$.}\label{preRec}
\end{figure}


Finally, for each size of sample $X_A$ we obtained 100 random experiments with corresponding values of $precision$ and $recall$. Figure \ref{preRec} shows the plots of the average $precision$ and $recall$ as functions from the initial sample size with 10\% and 90\% quantiles. 
In particular, Figure \ref{preRec} shows that the results are stable since the corridor between the 10\% and 90\% quantiles is narrow. The average value of $precision$ does not grow significantly, and the standard deviation decreases with the growth of the sample size. The average value of $recall$ grows with the sample size and its standard deviation decreases. Note that,  $precision$ reaches 90\% level within 100 elements in the labeled set, and $recall$ reaches this level within only 300 elements in the set $X_A$. It means, that in this experiment the high error in the estimating mean of the class $A$ leads to identifying only a subset of $A$ as a maximal by probability set.

In \cite{Scholkopf:2001:ESH:1119748.1119749} authors used the same dataset and they also tried to separate class of ``0'' from other digits. However, \cite{Scholkopf:2001:ESH:1119748.1119749} work in transformed space of kernel functions, whereas in our experiments we work in the original 256-dimensional space. 
The training set in their case consisted of all ``0'' images from 7291 images, which means that they had about 1200 images of ``0''. The test set contained 2007 images. 
They got $precision=91\%$ and $recall = 93\%$.
Figure \ref{preRec} shows that our approach reaches the same level of performance in terms of $precision$ and $recall$ with the initially labeled sample of 500 elements.


\section{Conclusion}\label{conclusion}

In this paper, we presented a novel semi-supervised one-class classification method which assumes that class is linearly separable from other elements. We proved theoretically that class $A$ is linearly separable if and only if it is maximal by probability within the sets with the same mean. Furthermore, we presented an algorithm for identifying such linearly separable class utilizing linear programming. We described three application cases including the assumption of linear separability of the class, Gaussian distribution, and the case of linear separability in the transformed space of kernel functions. 
Finally, we demonstrated the work of the proposed algorithm on the USPS dataset and analyzed the relationship of the performance of the algorithm and the size of the initially labeled sample.

As a future research, we plan to examine the work of the proposed algorithm in the real industrial settings.

\bibliographystyle{abbrv}
\bibliography{one_class_lib}

\end{document}